\documentclass{article}
\usepackage[T1]{fontenc}
\usepackage{iclr2027_conference,times}
\iclrfinalcopy

\usepackage{amsmath,amsfonts,bm}

\def\eqref#1{equation~\ref{#1}}

\def\1{\bm{1}}

\DeclareMathAlphabet{\mathsfit}{\encodingdefault}{\sfdefault}{m}{sl}
\SetMathAlphabet{\mathsfit}{bold}{\encodingdefault}{\sfdefault}{bx}{n}

\usepackage{algorithm}
\usepackage{algorithmic}
\usepackage{wrapfig}
\usepackage{amsthm}
\newtheorem{proposition}{Proposition}
\usepackage{booktabs}
\usepackage{threeparttable}
\usepackage{multirow}
\usepackage{colortbl}
\usepackage{arydshln}
\usepackage{bbding}
\usepackage{pifont}
\usepackage{enumitem}
\usepackage{needspace}

\usepackage{xcolor}
\definecolor{pink}{HTML}{747199}
\usepackage{hyperref}
\hypersetup{colorlinks=true, citecolor=pink, urlcolor=pink, linkcolor=pink}
\usepackage{url}
\usepackage{graphicx}
\usepackage{caption}
\usepackage{subcaption}
\usepackage{longtable}
\usepackage{tikz}
\usepackage[most]{tcolorbox}

\tcbuselibrary{breakable}
\tcbuselibrary{skins}
\newtcolorbox[
    use counter=tcboxcounter,number within=section
]{mybox}[3][]{
    left=3pt,
    right=4pt,
    breakable,
    enhanced,
    title=#2 \thetcbcounter: #3,
    #1
}
\newlength{\motivationfigureheight}

\title{Not All Criteria Matter: Importance-Aware Criteria Weighting for Rubric-Based LLM Evaluation}
\author{%
\begin{minipage}[t]{\dimexpr\textwidth-2\tabcolsep\relax}
\centering\normalfont
Haocheng Yang\textsuperscript{1}\quad
Licheng Pan\textsuperscript{2}\quad
Yiming Dong\textsuperscript{3}\quad
Xiaoxi Li\textsuperscript{4}\quad
Yuan Lu\textsuperscript{5}\\
Zhenhao Chen\textsuperscript{6}\quad
Jiajun Fan\textsuperscript{7}\quad
Shuai Shao\textsuperscript{8}\quad
Chunyuan Zheng\textsuperscript{9}\quad
Hao Wang\textsuperscript{6,}\thanks{Hao Wang is the corresponding author.}\\[0.8em]
{\small
\textsuperscript{1}National University of Singapore\quad
\textsuperscript{2}Zhejiang University\\
\textsuperscript{3}Alibaba Group\quad
\textsuperscript{4}Renmin University of China\quad
\textsuperscript{5}Xiaohongshu\\
\textsuperscript{6}Mohamed bin Zayed University of Artificial Intelligence\\
\textsuperscript{7}University of Illinois at Urbana-Champaign\\
\textsuperscript{8}Shanghai Jiaotong University\quad
\textsuperscript{9}Peking University}\\[0.4em]
{\small\texttt{haocheng\_yang@u.nus.edu, Ho-ward@outlook.com}}
\end{minipage}%
}

\begin{document}
\maketitle
\lhead{Preprint}

\begin{abstract}
Rubric-based evaluation assesses the responses of large language models (LLMs) by aggregating the verification results of explicit criteria, yet \textit{not all criteria matter equally}.
However, existing methods commonly aggregate the verification results with uniform weighting, which assigns minor and overlapping criteria the same weight as essential ones and thus produces misleading evaluation scores.
Although some methods prompt an LLM to directly assign weights, direct weighting requires a holistic judgment over all criteria that LLMs struggle to make reliably.
To address these limitations, we propose Importance-Aware Criteria Weighting (ICW), which derives the weights of criteria from pairwise importance comparisons made by an LLM.
Specifically, ICW transforms the holistic judgment over all criteria into a series of pairwise importance comparisons, checks and revises inconsistent comparisons, and derives the weights from the principal eigenvector of the comparison matrix.
Extensive experiments demonstrate that ICW consistently outperforms uniform weighting and competitive weighting methods, and translates into better policy models in downstream reinforcement learning.
Code is available at \url{https://github.com/yhc-666/Importance-Aware-Criteria-Weighting.git}.
\end{abstract}

\section{Introduction}
\label{sec:introduction}

Rubric-based evaluation, which assesses response quality against explicit criteria, has emerged as a prevalent paradigm for evaluating open-ended responses of large language models (LLMs)~\citep{hashemi2024llmrubric,li2025generation,zhang2026rubricbench}, with widespread applications in domains such as healthcare~\citep{arora2025healthbench}, deep research~\citep{du2025deepresearch}, and scholarly question answering~\citep{yifei2026researchqa}.
Specifically, an LLM-based verifier judges whether a response satisfies each criterion, and the verification results are aggregated into an evaluation score.

To improve rubric-based evaluation, existing works mainly follow two lines.
One line of work studies rubric generation, where criteria are written by experts~\citep{arora2025healthbench}, generated by LLMs for each user prompt~\citep{cook2024tick,li2026rubrichub}, and further refined to better distinguish between responses~\citep{zhang2026chasing,shen2026rethinking}.
A further line of work uses the evaluation score as a reward signal in reinforcement learning (RL) for open-ended tasks without verifiable answers~\citep{gunjal2026rubrics,viswanathan2025checklists,shan2026rubricrlsurvey}.

Despite the progress in both lines of work, existing methods commonly aggregate the verification results with uniform weighting~\citep{qin2024infobench,cook2024tick,lee2025checkeval}, which implicitly assumes that all criteria are equally important.
However, \textit{not all criteria matter equally}: some criteria are essential for addressing the user prompt, whereas others only concern minor aspects such as formatting or tone, and several criteria may describe the same property of a response.
By assigning minor and overlapping criteria the same weight as essential ones, uniform weighting can produce misleading evaluation scores that favor lower-quality responses.
While some methods prompt an LLM to directly assign weights to criteria~\citep{du2025deepresearch,gunjal2026rubrics,zhang2026chasing,ding2026adarubric}, direct weighting requires joint judgments over the rubric set and numerical weight assignments. Studies of LLM-based text evaluation report anchoring effects in multi-attribute judgments and inconsistent numerical ratings~\citep{stureborg2024large}.
These findings motivate examining alternatives to uniform weighting and direct numerical weight elicitation.

To this end, we propose Importance-Aware Criteria Weighting (ICW), which derives the weights of criteria from pairwise importance comparisons made by an LLM. This design is motivated by evidence for the effectiveness of pairwise judgments in text evaluation and ranking~\citep{liusie2024llm,liu2024aligning,qin2024large}.
Specifically, ICW transforms the holistic judgment over all criteria into a series of pairwise importance comparisons, each of which only requires the LLM to judge which of two criteria is more important for the user prompt, and derives the weights from the principal eigenvector of the resulting comparison matrix.
Since independent comparisons may contradict one another, ICW further checks the consistency of the comparisons and prompts the LLM to revise inconsistent comparisons.
Experimental results confirm that ICW produces more accurate evaluation scores and better policy models in downstream RL.

\paragraph{Contributions.} The contributions of this work are summarized as  follows.
\ding{182} \textbf{We identify the unequal importance of criteria as an under-addressed challenge in rubric-based evaluation.} Since not all criteria matter equally, uniform weighting can produce misleading evaluation scores.
\ding{183} \textbf{We propose ICW, which derives the weights of criteria from pairwise importance comparisons made by an LLM.} ICW derives the weights from the principal eigenvector of the comparison matrix, and checks and revises inconsistent comparisons.
\ding{184} \textbf{We conduct comprehensive experiments on six pairwise preference benchmarks to evaluate the efficacy of ICW,} showing  ICW consistently outperforms uniform weighting and competitive weighting methods and yields better policy models in downstream RL.

\section{Preliminaries}
\label{sec:preliminaries}

\subsection{Rubric-Based Evaluation}
\label{subsec:rubric_evaluation}

Rubric-based evaluation assesses response quality against explicit criteria~\citep{hashemi2024llmrubric,arora2025healthbench,du2025deepresearch}.
Given a user prompt $x$ and an LLM-generated response $y$, we define a rubric set $\mathcal{R}=\{c_1,\ldots,c_K\}$ comprising $K$ readily verifiable criteria.
Each criterion is a textual statement characterizing a property of an ideal response.
Suppose $v(\cdot)$ is an LLM-based verifier that outputs $v(x,y,c_k)\in\{0,1\}$ to indicate whether $y$ satisfies $c_k$ given $x$.
The evaluation score is defined as:
\begin{equation}
    s(x,y)=\sum_{k=1}^{K}w_k\,v(x,y,c_k),
    \label{eq:rubric_score}
\end{equation}
where $\boldsymbol{w}=(w_1,\ldots,w_K)$ denotes the weights reflecting the relative importance of the criteria, with $w_k\geq 0$ and $\sum_{k=1}^{K}w_k=1$.
By aggregating verification results across criteria, $s(x,y)$ provides an overall measure of response quality for downstream applications.
For data curation, $s(x,y)$ can be used to rank candidate responses and select high-quality examples~\citep{li2026rubrichub}; in RL~\citep{pan2026optimal,pan2026uncertainty,wang2026causalrm,wangunbiased}, it can serve as a reward signal to optimize the policy model~\citep{gunjal2026rubrics,viswanathan2025checklists,zhang2026chasing}.

\subsection{Problem Definition}
\label{subsec:problem_definition}

This work investigates the rubric weighting problem, which aims to estimate the weights $\boldsymbol{w}$ in Eq.~(\ref{eq:rubric_score}) that reflect the relative importance of the criteria.
Specifically, $\boldsymbol{w}$ is estimated from the user prompt $x$ and the fixed rubric set $\mathcal{R}$ alone, and is shared across candidate responses to aggregate their verification results.
\textbf{Thus, the core problem is to estimate $\boldsymbol{w}$ from $x$ and $\mathcal{R}$ such that $s(x,y)$ ranks candidate responses consistently with human preference.}

\section{Methodology}
\label{sec:methodology}

\subsection{Motivation}
\label{sec:motivation}

Rubric weighting is central to rubric-based evaluation.
Given the verification results, the weights $\boldsymbol{w}$ in Eq.~(\ref{eq:rubric_score}) determine how candidate responses are ranked, and thus which responses are selected in data curation~\citep{li2026rubrichub} and rewarded in RL~\citep{gunjal2026rubrics}.
Therefore, accurate rubric weighting is essential for reliable rubric-based evaluation and its downstream applications.

A simple solution to the rubric weighting problem is uniform weighting, which assigns $w_k=1/K$ to every criterion.
However, \textit{not all criteria matter equally.}
Specifically, some criteria are essential for addressing the user prompt, whereas others only concern minor aspects of a response, such as formatting or tone, and several criteria may describe the same property of a response.
By assigning minor and overlapping criteria the same weight as essential ones, uniform weighting allows several minor criteria to jointly outweigh an essential criterion and overlapping criteria to give repeated credit to the same property.
Consequently, a lower-quality response may receive a higher evaluation score than a higher-quality response.

\begin{figure}[htbp]
    \centering
    \captionsetup[subfigure]{font=normalsize,skip=2pt}
    \settoheight{\motivationfigureheight}{\includegraphics[width=0.49\linewidth]{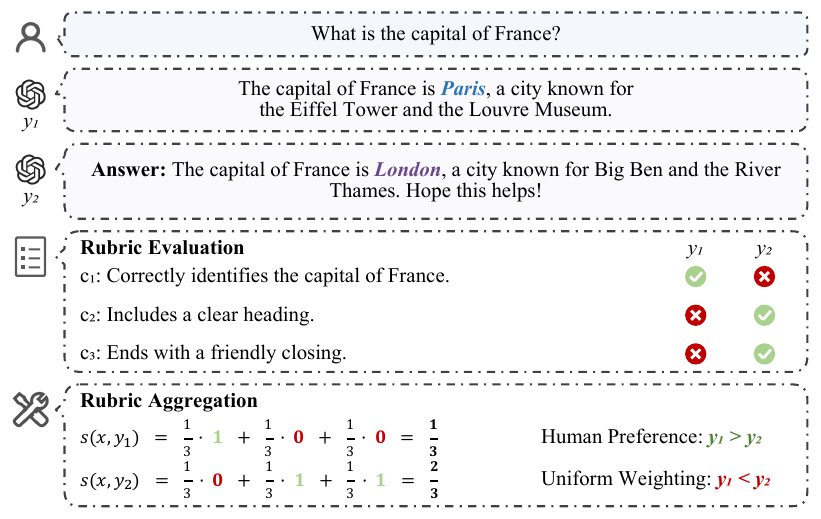}}
    \begin{subfigure}[t]{0.49\linewidth}
        \centering
        \includegraphics[width=\linewidth]{plots/case_study_unequal_importance.pdf}
        \caption{}
        \label{fig:motivation_importance}
    \end{subfigure}%
    \hfill
    \resizebox{!}{\motivationfigureheight}{%
        \begin{tikzpicture}
            \draw[dash pattern=on 14pt off 10.5pt, line width=2pt] (0,0) -- (0,8);
        \end{tikzpicture}%
    }%
    \hfill
    \begin{subfigure}[t]{0.49\linewidth}
        \centering
        \includegraphics[width=\linewidth]{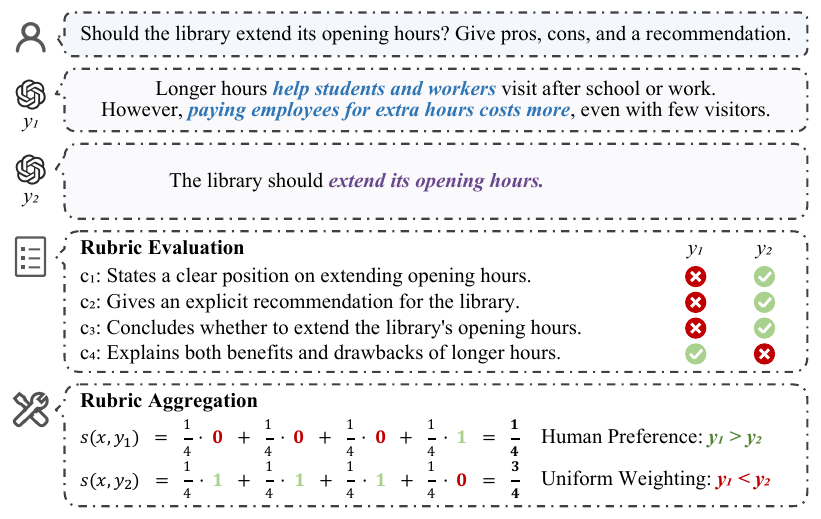}
        \caption{}
        \label{fig:motivation_overlap}
    \end{subfigure}
    \caption{Case studies in which uniform weighting assigns higher evaluation scores to lower-quality responses. (a) Minor stylistic criteria outweigh the essential criterion of factual correctness. (b) Overlapping criteria give repeated credit to the same recommendation.}
    \label{fig:motivation_case_studies}
\end{figure}

\textbf{Case studies.} To further substantiate the above limitation, we conduct case studies as follows.
In Figure~\ref{fig:motivation_case_studies}(\subref*{fig:motivation_importance}), several minor criteria outweigh an essential criterion: although $y_1$ correctly answers that the capital of France is Paris, the incorrect $y_2$ receives a higher evaluation score because $y_2$ includes a heading and a friendly closing.
In Figure~\ref{fig:motivation_case_studies}(\subref*{fig:motivation_overlap}), overlapping criteria give repeated credit to the same property: although $y_1$ explains both benefits and drawbacks of longer hours, the less informative $y_2$ receives a higher evaluation score because the recommendation of $y_2$ satisfies three overlapping criteria.
\textbf{Therefore, since not all criteria matter equally, uniform weighting can produce misleading evaluation scores that favor lower-quality responses.}

\begin{figure}[t]
    \centering
    \includegraphics[width=\linewidth]{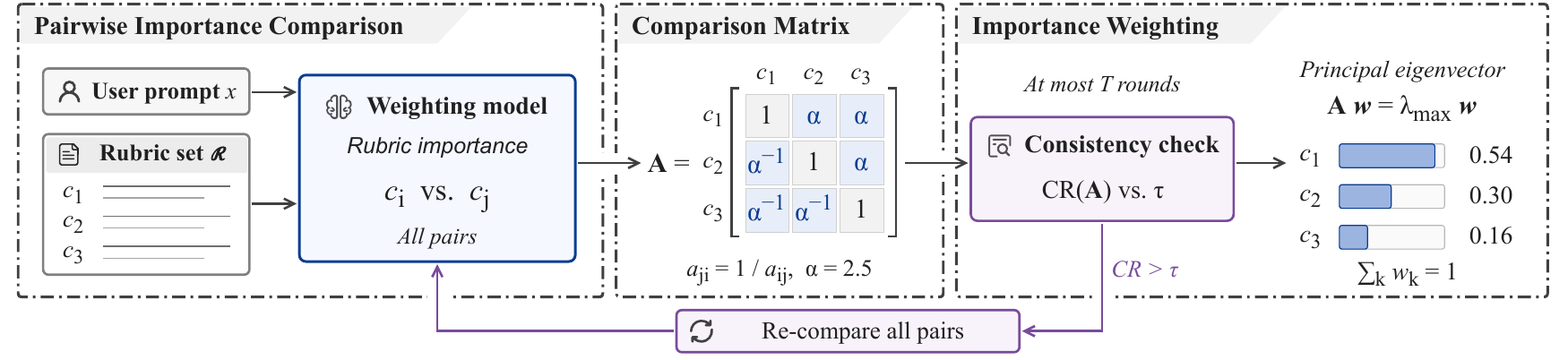}
    \caption{Overview of ICW. The weighting model compares the importance of every pair of criteria given the user prompt and the rubric set, and the judgments are mapped to a comparison matrix with the fixed ratio $\alpha$. If CR exceeds $\tau$, all pairs are re-compared for at most $T$ rounds, and a revised matrix replaces the current matrix only if its CR is strictly lower. The weights are obtained from the principal eigenvector of the current matrix.}
    \label{fig:method_overview}
\end{figure}

\subsection{Pairwise Importance Comparison}
\label{subsec:pairwise_comparison}

To account for the unequal importance of criteria, a natural approach is to prompt an LLM to directly assign numerical weights or importance levels to criteria~\citep{du2025deepresearch,viswanathan2025checklists,gunjal2026rubrics,zhang2026chasing,ding2026adarubric,li2026rubrichub}.
This approach requires the LLM to express its judgments over the rubric set as weights or importance levels in a single response.
Studies of LLM-based text evaluation report anchoring effects in multi-attribute judgments and inconsistent numerical ratings~\citep{stureborg2024large}.
Prior work also reports improvements from pairwise comparisons over absolute scoring in text evaluation~\citep{liusie2024llm,liu2024aligning}.
Motivated by these findings, we transform the holistic judgment over all criteria into a series of pairwise importance comparisons, each of which only requires judging which of two criteria is more important.

To this end, given the user prompt $x$ and the rubric set $\mathcal{R}$, we prompt an LLM, referred to as the weighting model, to judge which of two criteria $c_i,c_j\in\mathcal{R}$ is more important for evaluating responses to $x$, together with a brief reason.
The weighting model may also judge that neither criterion is more important when the two criteria overlap or are equally important.
Formally, we denote the judgment when $c_i$ is displayed before $c_j$ by $o_{ij}=\operatorname{Compare}(x,\mathcal{R},c_i,c_j)\in\{i,j,0\}$, where $o_{ij}=i$ or $o_{ij}=j$ indicates the more important criterion and $o_{ij}=0$ indicates that neither is more important.

However, LLM judgments are sensitive to the order in which two candidates are displayed~\citep{zheng2023judging}.
To mitigate this position bias, we query each pair in both display orders and regard a criterion as more important only when the two judgments agree.
We use a positive reciprocal comparison matrix~\citep{saaty1977scaling}. To populate it from ordinal judgments, we encode each accepted comparison with a fixed ratio $\alpha>1$, so the weighting model only needs to identify the more important criterion without estimating a numerical intensity.
Combining the two designs, we define the comparison of $c_i$ against $c_j$ as:
\begin{equation}
    a_{ij}=
    \begin{cases}
        \alpha, & \text{if } o_{ij}=o_{ji}=i,\\
        1/\alpha, & \text{if } o_{ij}=o_{ji}=j,\\
        1, & \text{otherwise},
    \end{cases}
    \label{eq:comparison}
\end{equation}
with $a_{ii}=1$.
In particular, $a_{ij}=1$ when neither criterion is judged more important or when the two display orders disagree, so that an order-dependent judgment favors neither criterion.
The resulting comparison matrix $\mathbf{A}=(a_{ij})\in\mathbb{R}_{>0}^{K\times K}$ is positive and reciprocal, i.e., $a_{ji}=1/a_{ij}$, and summarizes the relative importance between every pair of criteria for $x$.

\subsection{Importance Weighting via Principal Eigenvector}
\label{subsec:eigenvector_weighting}

In this section, we derive the weights $\boldsymbol{w}$ from the comparison matrix $\mathbf{A}$.
The key idea is that $a_{ij}$ specifies how the weight of $c_i$ should relate to that of $c_j$; for example, $a_{ij}=\alpha$ means that $c_i$ should receive $\alpha$ times the weight of $c_j$.
Therefore, each comparison should ideally equal the ratio between the weights of the two criteria, i.e., $a_{ij}=w_i/w_j$.
Since $(w_i/w_j)(w_j/w_\ell)=w_i/w_\ell$, the ratio $a_{ij}=w_i/w_j$ can hold for all pairs only if the direct comparison of any two criteria agrees with their indirect comparison through any other criterion.
Formally, we call $\mathbf{A}$ consistent if $a_{i\ell}=a_{ij}a_{j\ell}$ for all $i,j,\ell$.
Proposition~\ref{prop:eigenvector} shows that consistency is also sufficient for such weights to exist, and that the weights are then given by the principal eigenvector of $\mathbf{A}$.

\begin{proposition}[\citealp{saaty1977scaling}]
\label{prop:eigenvector}
A positive reciprocal matrix $\mathbf{A}\in\mathbb{R}_{>0}^{K\times K}$ is consistent if and only if there exist weights $\boldsymbol{w}$ with $w_k>0$ and $\sum_{k=1}^{K}w_k=1$ such that $a_{ij}=w_i/w_j$ for all $i,j$.
In this case, $\boldsymbol{w}$ is unique and is the principal eigenvector of $\mathbf{A}$ with eigenvalue $\lambda_{\max}=K$.
\end{proposition}

The proof is provided in Appendix~\ref{app:proofs}.
Therefore, it is natural to obtain the weights from the principal eigenvector of $\mathbf{A}$:
\begin{equation}
    \mathbf{A}\boldsymbol{w}=\lambda_{\max}\boldsymbol{w},
    \qquad \sum_{k=1}^{K}w_k=1,
    \qquad w_k>0,
    \label{eq:eigen_weights}
\end{equation}
where $\lambda_{\max}$ is the largest eigenvalue of $\mathbf{A}$.
By the Perron--Frobenius theorem~\citep{horn2012matrix}, $\boldsymbol{w}$ exists and is unique for any positive matrix, so the weights are well defined even when $\mathbf{A}$ is inconsistent.
Moreover, by Proposition~\ref{prop:eigenvector}, when the comparisons are consistent, a criterion judged more important than another receives $\alpha$ times the weight of the other criterion, and $\alpha=1$ recovers uniform weighting.
In this way, the weights aggregate the pairwise importance comparisons of all criteria, so that criteria judged more important receive larger weights and the evaluation score in Eq.~(\ref{eq:rubric_score}) no longer treats all criteria equally.

\subsection{Consistency Check and Revision}
\label{subsec:consistency_revision}

Although the principal eigenvector yields weights from any comparison matrix, the weights reproduce all comparisons, i.e., $a_{ij}=w_i/w_j$ for all $i,j$, only when $\mathbf{A}$ is consistent (Proposition~\ref{prop:eigenvector}).
Since each pair is compared separately, the comparisons may be inconsistent with each other.
For example, if $c_i$ is judged more important than $c_j$, $c_j$ more important than $c_\ell$, and $c_\ell$ more important than $c_i$, no weights can reproduce the three comparisons.
Note that since all accepted comparisons share the ratio $\alpha$, even transitive judgments can be inconsistent, e.g., $a_{ij}a_{j\ell}=\alpha^2\neq\alpha=a_{i\ell}$ when $c_i$, $c_j$, and $c_\ell$ are judged in decreasing importance.
Nevertheless, the principal eigenvector preserves the order of any transitive judgments (Proposition~\ref{prop:order}, Appendix~\ref{app:order}), whereas contradictory judgments such as the above cycle cannot be preserved by any weights; therefore, we use the degree of inconsistency as a relative measure of contradictory comparisons rather than requiring strict consistency.
In this section, we introduce consistency check and revision to detect and revise inconsistent comparisons.
The consistency check builds on Proposition~\ref{prop:consistency}, which characterizes consistency through the largest eigenvalue $\lambda_{\max}$.

\begin{proposition}[\citealp{saaty1977scaling}]
\label{prop:consistency}
For any positive reciprocal matrix $\mathbf{A}\in\mathbb{R}_{>0}^{K\times K}$, $\lambda_{\max}\geq K$, with equality if and only if $\mathbf{A}$ is consistent.
\end{proposition}

The proof is provided in Appendix~\ref{app:proofs}.
Therefore, the gap between $\lambda_{\max}$ and $K$ quantifies the inconsistency of $\mathbf{A}$.
Following~\citet{saaty1977scaling}, we normalize the gap into the consistency index (CI), which can be rewritten as (see Appendix~\ref{app:proofs}):
\begin{equation}
    \mathrm{CI}(\mathbf{A})=\frac{\lambda_{\max}-K}{K-1}=\frac{1}{K(K-1)}\sum_{i<j}\left(\epsilon_{ij}+\frac{1}{\epsilon_{ij}}-2\right),
    \qquad \epsilon_{ij}=\frac{a_{ij}w_j}{w_i},
    \label{eq:consistency_index}
\end{equation}
where $\boldsymbol{w}$ is the principal eigenvector in Eq.~(\ref{eq:eigen_weights}).
Since $\epsilon_{ij}+1/\epsilon_{ij}-2\geq0$ with equality if and only if $a_{ij}=w_i/w_j$, CI aggregates the deviations between the comparisons and the ratios of the obtained weights over all pairs, i.e., CI measures how well the weights reproduce the comparisons.
Since the CI of random comparisons grows with $K$, we further adopt the consistency ratio (CR), a heuristic normalization of CI across rubric sets of different sizes:
\begin{equation}
    \mathrm{CR}(\mathbf{A})=\frac{\mathrm{CI}(\mathbf{A})}{\mathrm{RI}_K},
    \label{eq:consistency}
\end{equation}
where $\mathrm{RI}_K$ is the average CI of randomly generated positive reciprocal matrices of size $K$~\citep{saaty1980analytic} (see Appendix~\ref{app:implementation} for $K>15$).\footnote{For $K\le2$, every positive reciprocal matrix is consistent, and we set $\mathrm{CR}(\mathbf{A})=0$.}
By Proposition~\ref{prop:consistency}, $\mathrm{CR}(\mathbf{A})=0$ if and only if $\mathbf{A}$ is consistent.
We accept $\mathbf{A}$ if $\mathrm{CR}(\mathbf{A})\le\tau$, where $\tau$ is the consistency threshold.

Otherwise, we revise the comparisons for at most $T$ rounds, starting from $\mathbf{A}$ as the current matrix.
In each round, the weighting model re-compares every pair in both display orders given the judgments and reasons of the previous round, denoted by $\mathcal{H}$, i.e., $o'_{ij}=\operatorname{Compare}(x,\mathcal{R},c_i,c_j;\mathcal{H})$, so that each pair is re-examined in light of the judgments of the other pairs.
The revised judgments form a revised matrix $\mathbf{A}'$ via Eq.~(\ref{eq:comparison}), which replaces the current matrix $\mathbf{A}$ only if $\mathrm{CR}(\mathbf{A}')<\mathrm{CR}(\mathbf{A})$.
The revision stops once $\mathrm{CR}(\mathbf{A})\le\tau$ or after $T$ rounds, and the weights are obtained from the current matrix $\mathbf{A}$ via Eq.~(\ref{eq:eigen_weights}).
In this way, the revision never increases the inconsistency of the comparisons.

\subsection{The Workflow of ICW}
\label{subsec:workflow}

\begin{algorithm}[t]
\caption{ICW}
\label{alg:icw}
\small
{\raggedright
\textbf{Input}: user prompt $x$; rubric set $\mathcal{R}=\{c_1,\ldots,c_K\}$.\par
\textbf{Parameter}: fixed ratio $\alpha$; consistency threshold $\tau$; maximum rounds $T$.\par
\textbf{Output}: weights $\boldsymbol{w}$.\par}
\setboolean{ALC@noend}{true}
\begin{algorithmic}[1]
\FOR{each pair $(c_i,c_j)$ with $i<j$}
    \STATE $o_{ij}\leftarrow\operatorname{Compare}(x,\mathcal{R},c_i,c_j)$
    \STATE $o_{ji}\leftarrow\operatorname{Compare}(x,\mathcal{R},c_j,c_i)$
\ENDFOR
\STATE Construct $\mathbf{A}$ via Eq.~(\ref{eq:comparison}); $t\leftarrow0$
\WHILE{$\mathrm{CR}(\mathbf{A})>\tau$ \textbf{and} $t<T$}
    \STATE $t\leftarrow t+1$
    \STATE Re-compare all pairs given $\mathcal{H}$ (steps 1--3)
    \STATE Construct $\mathbf{A}'$ via Eq.~(\ref{eq:comparison})
    \STATE \textbf{if} $\mathrm{CR}(\mathbf{A}')<\mathrm{CR}(\mathbf{A})$ \textbf{then} $\mathbf{A}\leftarrow\mathbf{A}'$
\ENDWHILE
\STATE Obtain $\boldsymbol{w}$ from $\mathbf{A}$ via Eq.~(\ref{eq:eigen_weights})
\STATE \textbf{return} $\boldsymbol{w}$
\end{algorithmic}
\end{algorithm}

In this section, we present the workflow of ICW, as illustrated in Figure~\ref{fig:method_overview} and Algorithm~\ref{alg:icw}.
First, the weighting model compares every pair of criteria in both display orders (steps 1--3), and the judgments are mapped to the comparison matrix $\mathbf{A}$ via Eq.~(\ref{eq:comparison}) (step 4).
Then, while CR exceeds $\tau$, all pairs are re-compared given the judgments and reasons of the previous round for at most $T$ rounds, and each revised matrix replaces the current matrix only if its CR is strictly lower (steps 5--9).
Finally, the weights are obtained from the principal eigenvector of the current matrix via Eq.~(\ref{eq:eigen_weights}) (step 10).
Each round involves $K(K-1)$ queries to the weighting model, which are independent of each other and can be issued in parallel.
Since $\boldsymbol{w}$ depends only on $x$ and $\mathcal{R}$, it is computed once for each user prompt and shared across all candidate responses, so that the cost of ICW does not grow with the number of candidate responses.

\section{Experiments}
\label{sec:experiments}

In this section, we conduct a systematic empirical evaluation to demonstrate the efficacy of ICW, covering five research questions (RQs):
\begin{enumerate}[leftmargin=*]
    \item \textbf{RQ1}: \textit{Does ICW perform well?} In Section~\ref{subsec:main_results}, we compare ICW against uniform weighting and competitive weighting methods on six benchmarks.
    \item \textbf{RQ2}: \textit{Why does it work?} In Section~\ref{subsec:ablation}, we perform an ablation study on the contribution of pairwise importance comparison and consistency check and revision.
    \item \textbf{RQ3}: \textit{Does it generalize across verifiers and weighting models?} In Section~\ref{subsec:generalization}, we evaluate its compatibility with different verifiers and weighting models.
    \item \textbf{RQ4}: \textit{Is it sensitive to hyperparameters?} In Section~\ref{subsec:hyperparameter}, we analyze its sensitivity to $\alpha$ and $\tau$.
    \item \textbf{RQ5}: \textit{Does it improve downstream RL?} In Section~\ref{subsec:downstream_rl}, we validate whether the improved evaluation scores translate into better policy models.
\end{enumerate}

\subsection{Experimental Setup}
\label{subsec:setup}

\textbf{Datasets and Rubrics.}
We evaluate on six pairwise preference benchmarks, where each sample consists of a user prompt, two candidate responses, and a preference label: RubricBench~\citep{zhang2026rubricbench}, ResearchQA~\citep{yifei2026researchqa}, JudgeBench~\citep{tan2025judgebench}, the Chat subset of RM-Bench~\citep{liu2025rmbench} (RM-Chat), the Chat-Hard subset of RewardBench~\citep{lambert2025rewardbench} (Chat-Hard), and the Precise-IF subset of RewardBench~2~\citep{malik2026rewardbench2} (Precise-IF).
For RubricBench and ResearchQA, we use the rubric sets provided by the datasets.
For the other four benchmarks, we use the rubric sets generated by the rubric generator of Rubric-ARM~\citep{xu2026rubricarm}.
The weights provided by the rubric sources are removed, so that all methods weight the same criteria for each user prompt.
More details are provided in Appendix~\ref{app:dataset}.

\textbf{Baselines.}
We compare ICW against diverse rubric weighting methods, grouped into:
(1) \textbf{Direct Weighting Methods}, which prompt an LLM to assign weights or importance levels to all criteria at once, including Direct, AdaRubric~\citep{ding2026adarubric}, RaR~\citep{gunjal2026rubrics}, RACE-SC~\citep{du2025deepresearch}, and ChasingTail~\citep{zhang2026chasing};
and (2) \textbf{Indirect Weighting Methods}, which derive the weights from pairwise comparisons or verification results, including MAPLE~\citep{jinji2026maple} and RRD-WU~\citep{shen2026rethinking}, respectively.
Except for Direct, these methods are originally proposed within larger pipelines, such as rubric generation and policy training, and we adopt only their weighting components.
We further report Uniform, which assigns equal weights to all criteria, as a reference.

\textbf{Evaluation Metric.}
For each sample, a method predicts the candidate response with the higher evaluation score in Eq.~(\ref{eq:rubric_score}) as the preferred one, where a tie earns half credit.
We report the pairwise preference accuracy (Acc), together with the relative improvement (RI) over Uniform.

\textbf{Implementation Details.}
For fair comparison, all methods share the same rubric sets and verification results, and differ only in the weights.
Unless otherwise specified, we use GPT-4.1 as the weighting model for ICW and the LLM-based baselines, and Gemini-3.1-Flash-Lite as the verifier.
Both the weighting model and the verifier are queried with temperature 0.
For ICW, we set the maximum number of revision rounds to $T=1$, and select the fixed ratio $\alpha$ and the consistency threshold $\tau$ for each benchmark.

\subsection{Overall Performance (RQ1)}
\label{subsec:main_results}

\begin{table}[t]
\centering
\begin{threeparttable}
\caption{Pairwise preference accuracy of ICW versus baselines.}
\label{tab:main_result}
\setlength\tabcolsep{4.57pt}
\renewcommand\arraystretch{1.1}
\scriptsize
\begin{tabular}{lcccccccccccc}
\toprule
 & \multicolumn{2}{c}{\textbf{RubricBench}}
 & \multicolumn{2}{c}{\textbf{ResearchQA}}
 & \multicolumn{2}{c}{\textbf{JudgeBench}}
 & \multicolumn{2}{c}{\textbf{RM-Chat}}
 & \multicolumn{2}{c}{\textbf{Chat-Hard}}
 & \multicolumn{2}{c}{\textbf{Precise-IF}} \\
\cmidrule(lr){2-3} \cmidrule(lr){4-5} \cmidrule(lr){6-7} \cmidrule(lr){8-9} \cmidrule(lr){10-11} \cmidrule(lr){12-13}
\textbf{Method}
 & Acc & RI & Acc & RI & Acc & RI & Acc & RI & Acc & RI & Acc & RI \\
\midrule
Uniform & 0.771 & -- & 0.583 & -- & 0.762 & -- & 0.740 & -- & 0.805 & -- & 0.720 & -- \\
Direct & 0.778 & 0.96\% $\uparrow$ & 0.589 & 1.03\% $\uparrow$ & 0.770 & 1.06\% $\uparrow$ & 0.731 & 1.22\% $\downarrow$ & 0.802 & 0.41\% $\downarrow$ & 0.723 & 0.43\% $\uparrow$ \\
AdaRubric & 0.774 & 0.45\% $\uparrow$ & 0.592 & 1.53\% $\uparrow$ & 0.765 & 0.31\% $\uparrow$ & 0.749 & 1.23\% $\uparrow$ & 0.812 & 0.96\% $\uparrow$ & 0.732 & 1.74\% $\uparrow$ \\
RaR & 0.782 & 1.41\% $\uparrow$ & 0.604 & 3.57\% $\uparrow$ & 0.763 & 0.10\% $\uparrow$ & 0.760 & 2.62\% $\uparrow$ & 0.816 & 1.37\% $\uparrow$ & 0.746 & 3.61\% $\uparrow$ \\
MAPLE & 0.782 & 1.41\% $\uparrow$ & 0.601 & 3.05\% $\uparrow$ & 0.760 & 0.21\% $\downarrow$ & 0.745 & 0.70\% $\uparrow$ & 0.812 & 0.96\% $\uparrow$ & 0.738 & 2.60\% $\uparrow$ \\
RRD-WU & 0.780 & 1.25\% $\uparrow$ & 0.610 & 4.58\% $\uparrow$ & 0.772 & 1.27\% $\uparrow$ & 0.760 & 2.62\% $\uparrow$ & 0.815 & 1.23\% $\uparrow$ & 0.744 & 3.33\% $\uparrow$ \\
ChasingTail & 0.782 & 1.47\% $\uparrow$ & 0.604 & 3.57\% $\uparrow$ & 0.770 & 1.06\% $\uparrow$ & 0.760 & 2.62\% $\uparrow$ & 0.814 & 1.09\% $\uparrow$ & 0.743 & 3.18\% $\uparrow$ \\
RACE-SC & 0.780 & 1.19\% $\uparrow$ & 0.601 & 3.05\% $\uparrow$ & 0.768 & 0.84\% $\uparrow$ & 0.757 & 2.27\% $\uparrow$ & 0.810 & 0.68\% $\uparrow$ & 0.756 & 5.07\% $\uparrow$ \\
\rowcolor[HTML]{ecf0ff}
\textbf{ICW} & \textbf{0.795} & \textbf{3.11\% $\uparrow$} & \textbf{0.612} & \textbf{5.08\% $\uparrow$} & \textbf{0.783} & \textbf{2.76\% $\uparrow$} & \textbf{0.769} & \textbf{3.84\% $\uparrow$} & \textbf{0.818} & \textbf{1.64\% $\uparrow$} & \textbf{0.762} & \textbf{5.79\% $\uparrow$} \\
\bottomrule
\end{tabular}
\begin{tablenotes}
    \scriptsize
    \item \textit{Note}: ``RI'' denotes the relative improvement over the Uniform baseline.
\end{tablenotes}
\end{threeparttable}
\end{table}

In this section, we compare ICW against uniform weighting and competitive weighting methods on six benchmarks.
The results are presented in Table~\ref{tab:main_result} with key observations as follows.
\ding{182} \textbf{Uniform weighting leads to inferior evaluation accuracy.}
Uniform underperforms nearly all weighting methods across benchmarks.
This indicates that treating all criteria equally produces misleading evaluation scores, since not all criteria matter equally.
\ding{183} \textbf{Direct weighting methods offer limited and inconsistent improvements.}
Direct only marginally improves over Uniform and even underperforms Uniform on RM-Chat and Chat-Hard, while stabilizing the direct weights by averaging repeated generations (RACE-SC) or assigning coarse importance levels (RaR and ChasingTail) brings larger improvements.
This indicates that the holistic judgment over all criteria is unreliable, which limits the accuracy of direct weighting.
\ding{184} \textbf{Indirect weighting methods partially address the limitation of uniform weighting.}
RRD-WU achieves the strongest baseline performance by downweighting correlated criteria.
However, it disregards the relative importance of the criteria, leading to suboptimal performance.
\ding{185} \textbf{ICW achieves state-of-the-art performance across all benchmarks.}
ICW consistently outperforms all baselines, improving the accuracy over Uniform by up to 5.79\%.
This success is attributed to pairwise importance comparison and consistency check and revision, which together yield weights that reflect the relative importance of the criteria.

\subsection{Ablation Study (RQ2)}
\label{subsec:ablation}

\begin{table}[t]
\centering
\caption{Ablation study results of the two ICW components.}
\label{tab:ablation}
\setlength{\tabcolsep}{3.39pt}
\renewcommand\arraystretch{1}
\scriptsize
\begin{threeparttable}
\begin{tabular}{lcccccccccccccc}
\toprule
\multirow{2}{*}{\textbf{Method}} & \multirow{2}{*}{\textbf{Pair.}} & \multirow{2}{*}{\textbf{Rev.}}
& \multicolumn{2}{c}{\textbf{RubricBench}}
& \multicolumn{2}{c}{\textbf{ResearchQA}}
& \multicolumn{2}{c}{\textbf{JudgeBench}}
& \multicolumn{2}{c}{\textbf{RM-Chat}}
& \multicolumn{2}{c}{\textbf{Chat-Hard}}
& \multicolumn{2}{c}{\textbf{Precise-IF}} \\
\cmidrule(lr){4-5} \cmidrule(lr){6-7} \cmidrule(lr){8-9} \cmidrule(lr){10-11} \cmidrule(lr){12-13} \cmidrule(lr){14-15}
&&& Acc & RI & Acc & RI & Acc & RI & Acc & RI & Acc & RI & Acc & RI \\
\midrule
Uniform & \XSolidBrush & \XSolidBrush & 0.771 & -- & 0.583 & -- & 0.762 & -- & 0.740 & -- & 0.805 & -- & 0.720 & -- \\
ICW$^\dagger$ & \XSolidBrush & \Checkmark & 0.782 & 1.47\% $\uparrow$ & 0.592 & 1.53\% $\uparrow$ & 0.758 & 0.52\% $\downarrow$ & 0.747 & 0.88\% $\uparrow$ & 0.803 & 0.27\% $\downarrow$ & 0.755 & 4.92\% $\uparrow$ \\
ICW$^\ddagger$ & \Checkmark & \XSolidBrush & 0.790 & 2.49\% $\uparrow$ & 0.607 & 4.07\% $\uparrow$ & \textbf{0.783} & \textbf{2.76\% $\uparrow$} & 0.758 & 2.44\% $\uparrow$ & 0.804 & 0.14\% $\downarrow$ & 0.752 & 4.49\% $\uparrow$ \\
\rowcolor[HTML]{ecf0ff}
\textbf{ICW} & \Checkmark & \Checkmark & \textbf{0.795} & \textbf{3.11\% $\uparrow$} & \textbf{0.612} & \textbf{5.08\% $\uparrow$} & \textbf{0.783} & \textbf{2.76\% $\uparrow$} & \textbf{0.769} & \textbf{3.84\% $\uparrow$} & \textbf{0.818} & \textbf{1.64\% $\uparrow$} & \textbf{0.762} & \textbf{5.79\% $\uparrow$} \\
\bottomrule
\end{tabular}
\begin{tablenotes}
\item \scriptsize \textit{Note:} ``Pair.'' denotes pairwise importance comparison, and ``Rev.'' denotes consistency check and revision.
\end{tablenotes}
\end{threeparttable}
\end{table}

In this section, we dissect the individual contributions of the two designs of ICW: pairwise importance comparison, and consistency check and revision.
The results are presented in Table~\ref{tab:ablation} with key observations as follows.
\ding{182} \textbf{Pairwise importance comparison yields the dominant performance gain.}
Specifically, in ICW$^\dagger$, we compare all criterion pairs within a single query for each display order, while retaining consistency check and revision.
We observe a consistent performance drop across all benchmarks (e.g., the accuracy decreases from 0.783 to 0.758 on JudgeBench).
This confirms that comparing all pairs at once reintroduces the holistic judgment over all criteria, which is less reliable than judging which of two criteria is more important.
\ding{183} \textbf{Consistency check and revision yields essential complementary gains.}
Specifically, in ICW$^\ddagger$, we remove consistency check and revision, and obtain the weights from the initial comparisons.
We also observe a performance drop on nearly all benchmarks (e.g., the accuracy decreases from 0.818 to 0.804 on Chat-Hard).
This is consistent with Section~\ref{subsec:consistency_revision}: revising inconsistent comparisons allows the weights to better reflect the relative importance of the criteria.
\ding{184} \textbf{The benefits from both designs can be synergistically integrated.}
This is evidenced by the complete ICW achieving the best accuracy on all benchmarks.

\subsection{Generalization Studies (RQ3)}
\label{subsec:generalization}

\begin{table}[t]
\begin{minipage}[t]{0.49\linewidth}
\centering
\caption{Results across verifiers.}
\label{tab:gen_verifier}
\belowrulesep=1pt
\aboverulesep=0.5pt
\setlength{\tabcolsep}{3.44pt}
\renewcommand{\arraystretch}{1}
\scriptsize
\begin{threeparttable}
\begin{tabular}{lcccccc}
\toprule
\multirow{2}{*}{\textbf{Method}} & \multicolumn{2}{c}{\textbf{RM-Chat}} & \multicolumn{2}{c}{\textbf{Chat-Hard}} & \multicolumn{2}{c}{\textbf{Precise-IF}} \\
\cmidrule(lr){2-3} \cmidrule(lr){4-5} \cmidrule(lr){6-7}
& Acc & RI & Acc & RI & Acc & RI \\
\midrule
\rowcolor[HTML]{f0f0f0}
\multicolumn{7}{l}{\textbf{\textit{Verifier: Gemini-3.5-Flash}}} \\
\rowcolor{white}
Uniform & 0.739 & -- & 0.777 & -- & 0.728 & -- \\
\rowcolor[HTML]{ecf0ff}
\textbf{ICW} & \textbf{0.742} & \textbf{0.35\% $\uparrow$} & \textbf{0.819} & \textbf{5.36\% $\uparrow$} & \textbf{0.767} & \textbf{5.30\% $\uparrow$} \\
\midrule
\rowcolor[HTML]{f0f0f0}
\multicolumn{7}{l}{\textbf{\textit{Verifier: GPT-4.1-mini}}} \\
\rowcolor{white}
Uniform & 0.668 & -- & 0.794 & -- & 0.647 & -- \\
\rowcolor[HTML]{ecf0ff}
\textbf{ICW} & \textbf{0.673} & \textbf{0.76\% $\uparrow$} & \textbf{0.810} & \textbf{2.07\% $\uparrow$} & \textbf{0.655} & \textbf{1.28\% $\uparrow$} \\
\midrule
\rowcolor[HTML]{f0f0f0}
\multicolumn{7}{l}{\textbf{\textit{Verifier: DeepSeek-V4-Flash}}} \\
\rowcolor{white}
Uniform & 0.703 & -- & 0.788 & -- & 0.719 & -- \\
\rowcolor[HTML]{ecf0ff}
\textbf{ICW} & \textbf{0.715} & \textbf{1.66\% $\uparrow$} & \textbf{0.812} & \textbf{3.06\% $\uparrow$} & \textbf{0.748} & \textbf{4.05\% $\uparrow$} \\
\midrule
\rowcolor[HTML]{f0f0f0}
\multicolumn{7}{l}{\textbf{\textit{Verifier: GLM-5.1}}} \\
\rowcolor{white}
Uniform & 0.712 & -- & 0.774 & -- & 0.734 & -- \\
\rowcolor[HTML]{ecf0ff}
\textbf{ICW} & \textbf{0.735} & \textbf{3.26\% $\uparrow$} & \textbf{0.820} & \textbf{5.96\% $\uparrow$} & \textbf{0.769} & \textbf{4.68\% $\uparrow$} \\
\midrule
\rowcolor[HTML]{f0f0f0}
\multicolumn{7}{l}{\textbf{\textit{Verifier: Gemini-3.1-Flash-Lite}}} \\
\rowcolor{white}
Uniform & 0.740 & -- & 0.805 & -- & 0.720 & -- \\
\rowcolor[HTML]{ecf0ff}
\textbf{ICW} & \textbf{0.769} & \textbf{3.84\% $\uparrow$} & \textbf{0.818} & \textbf{1.64\% $\uparrow$} & \textbf{0.762} & \textbf{5.79\% $\uparrow$} \\
\bottomrule
\end{tabular}
\begin{tablenotes}
\item \tiny \textit{Note}: ``RI'' denotes relative improvement over the Uniform baseline.
\end{tablenotes}
\end{threeparttable}
\end{minipage}
\hfill
\begin{minipage}[t]{0.49\linewidth}
\centering
\caption{Results across weighting models.}
\label{tab:gen_weighting_model}
\belowrulesep=1pt
\aboverulesep=0.5pt
\setlength{\tabcolsep}{3.44pt}
\renewcommand{\arraystretch}{1}
\scriptsize
\begin{threeparttable}
\begin{tabular}{lcccccc}
\toprule
\multirow{2}{*}{\textbf{Method}} & \multicolumn{2}{c}{\textbf{RM-Chat}} & \multicolumn{2}{c}{\textbf{Chat-Hard}} & \multicolumn{2}{c}{\textbf{Precise-IF}} \\
\cmidrule(lr){2-3} \cmidrule(lr){4-5} \cmidrule(lr){6-7}
& Acc & RI & Acc & RI & Acc & RI \\
\midrule
\rowcolor[HTML]{f0f0f0}
\multicolumn{7}{l}{\textbf{\textit{Weighting Model: Gemini-3.5-Flash}}} \\
\rowcolor{white}
Uniform & 0.740 & -- & 0.805 & -- & 0.720 & -- \\
\rowcolor[HTML]{ecf0ff}
\textbf{ICW} & \textbf{0.758} & \textbf{2.44\% $\uparrow$} & \textbf{0.823} & \textbf{2.32\% $\uparrow$} & \textbf{0.770} & \textbf{6.95\% $\uparrow$} \\
\midrule
\rowcolor[HTML]{f0f0f0}
\multicolumn{7}{l}{\textbf{\textit{Weighting Model: GPT-4.1-mini}}} \\
\rowcolor{white}
Uniform & 0.740 & -- & 0.805 & -- & 0.720 & -- \\
\rowcolor[HTML]{ecf0ff}
\textbf{ICW} & \textbf{0.758} & \textbf{2.44\% $\uparrow$} & \textbf{0.820} & \textbf{1.91\% $\uparrow$} & \textbf{0.750} & \textbf{4.20\% $\uparrow$} \\
\midrule
\rowcolor[HTML]{f0f0f0}
\multicolumn{7}{l}{\textbf{\textit{Weighting Model: DeepSeek-V4-Flash}}} \\
\rowcolor{white}
Uniform & 0.740 & -- & 0.805 & -- & 0.720 & -- \\
\rowcolor[HTML]{ecf0ff}
\textbf{ICW} & \textbf{0.760} & \textbf{2.62\% $\uparrow$} & \textbf{0.818} & \textbf{1.64\% $\uparrow$} & \textbf{0.767} & \textbf{6.52\% $\uparrow$} \\
\midrule
\rowcolor[HTML]{f0f0f0}
\multicolumn{7}{l}{\textbf{\textit{Weighting Model: GLM-5.1}}} \\
\rowcolor{white}
Uniform & 0.740 & -- & 0.805 & -- & 0.720 & -- \\
\rowcolor[HTML]{ecf0ff}
\textbf{ICW} & \textbf{0.756} & \textbf{2.09\% $\uparrow$} & \textbf{0.822} & \textbf{2.19\% $\uparrow$} & \textbf{0.768} & \textbf{6.65\% $\uparrow$} \\
\midrule
\rowcolor[HTML]{f0f0f0}
\multicolumn{7}{l}{\textbf{\textit{Weighting Model: GPT-4.1}}} \\
\rowcolor{white}
Uniform & 0.740 & -- & 0.805 & -- & 0.720 & -- \\
\rowcolor[HTML]{ecf0ff}
\textbf{ICW} & \textbf{0.769} & \textbf{3.84\% $\uparrow$} & \textbf{0.818} & \textbf{1.64\% $\uparrow$} & \textbf{0.762} & \textbf{5.79\% $\uparrow$} \\
\bottomrule
\end{tabular}
\begin{tablenotes}
\item \tiny \textit{Note}: ``RI'' denotes relative improvement over the Uniform baseline.
\end{tablenotes}
\end{threeparttable}
\end{minipage}
\end{table}

In this section, we evaluate the generalizability of ICW across verifiers and weighting models from different model families, including Gemini, GPT, DeepSeek, and GLM.
The results are presented in Tables~\ref{tab:gen_verifier} and~\ref{tab:gen_weighting_model}, and the results on the other three benchmarks are provided in Appendix~\ref{app:generalization}, with key observations as follows.
\ding{182} \textbf{ICW generalizes across verifiers.}
With GPT-4.1 as the weighting model, ICW generally outperforms Uniform across the five verifiers.
This indicates that the weights of ICW complement the verification results rather than depend on a particular verifier.
\ding{183} \textbf{ICW generalizes across weighting models.}
With Gemini-3.1-Flash-Lite as the verifier, ICW consistently outperforms Uniform with all five weighting models, including the smaller GPT-4.1-mini.
This shows that pairwise importance comparison does not rely on a particular weighting model.
Collectively, these findings establish ICW as a model-agnostic rubric weighting method.

\subsection{Hyperparameter Sensitivity Study (RQ4)}
\label{subsec:hyperparameter}

\begin{wrapfigure}[9]{r}{0.5\linewidth}
\vspace{-2em}
\centering
\includegraphics[width=\linewidth]{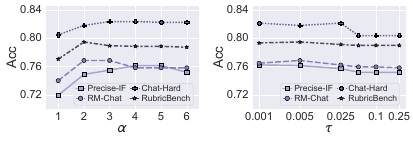}
\vspace{-1em}
\caption{Sensitivity to the fixed ratio $\alpha$ and the consistency threshold $\tau$ with GPT-4.1 as the weighting model.}
\label{fig:sensitivity}
\vspace{-5pt}
\end{wrapfigure}

In this section, we examine the impact of the fixed ratio $\alpha$ and the consistency threshold $\tau$ on ICW.
The results are presented in Figure~\ref{fig:sensitivity}, and the results with other weighting models are provided in Appendix~\ref{app:sensitivity}, with key observations as follows.
\ding{182} \textbf{ICW is robust to the fixed ratio.}
ICW substantially improves over uniform weighting ($\alpha=1$) once $\alpha\geq2$, and the accuracy remains stable as $\alpha$ further increases.
Notably, with any tested $\alpha\geq2$, ICW outperforms the strongest baseline RRD-WU in the average accuracy over the four benchmarks in Figure~\ref{fig:sensitivity}.
This indicates that the advantage of ICW stems from pairwise importance comparison rather than a delicate choice of $\alpha$.
\ding{183} \textbf{Smaller consistency thresholds improve accuracy by triggering more revisions.}
As $\tau$ decreases, more comparison matrices are revised and the accuracy generally improves, whereas ICW approaches its variant without revision for large $\tau$.
This is consistent with the ablation results in Section~\ref{subsec:ablation}.

\subsection{Downstream RL Performance (RQ5)}
\label{subsec:downstream_rl}

\begin{table}[t]
\centering
\caption{Downstream RL results.}
\label{tab:downstream_rl}
\setlength{\tabcolsep}{6.61pt}
\renewcommand{\arraystretch}{1}
\scriptsize
\begin{threeparttable}
\begin{tabular}{lcccccccccc}
\toprule
\multirow{2.5}{*}{\textbf{Method}} & \multicolumn{2}{c}{\textbf{BiGGen Bench}} & \multicolumn{4}{c}{\textbf{AlpacaEval 2.0}} & \multicolumn{4}{c}{\textbf{Arena-Hard}} \\
\cmidrule(lr){2-3} \cmidrule(lr){4-7} \cmidrule(lr){8-11}
& Score & RI & WR & RI & LC & RI & WR & RI & SC & RI \\
\midrule
\rowcolor[HTML]{f0f0f0}
\multicolumn{11}{l}{\textit{\textbf{Policy Model: Qwen2.5-3B-Instruct}}} \\
\rowcolor{white}
Base Model & 59.22 & -- & 17.25 & -- & 16.71 & -- & 25.93 & -- & 27.63 & -- \\
Uniform & 61.45 & 3.77\% $\uparrow$ & 21.59 & 25.16\% $\uparrow$ & 19.07 & 14.12\% $\uparrow$ & 27.58 & 6.36\% $\uparrow$ & 28.58 & 3.44\% $\uparrow$ \\
RACE-SC & \textbf{64.95} & \textbf{9.68\% $\uparrow$} & 24.32 & 40.99\% $\uparrow$ & 22.74 & 36.09\% $\uparrow$ & 28.67 & 10.57\% $\uparrow$ & 29.61 & 7.17\% $\uparrow$ \\
RRD-WU & 64.23 & 8.46\% $\uparrow$ & 24.37 & 41.28\% $\uparrow$ & 20.42 & 22.20\% $\uparrow$ & 27.78 & 7.13\% $\uparrow$ & 29.80 & 7.85\% $\uparrow$ \\
\rowcolor[HTML]{ecf0ff}
\textbf{ICW} & 64.78 & 9.39\% $\uparrow$ & \textbf{26.23} & \textbf{52.06\% $\uparrow$} & \textbf{23.54} & \textbf{40.87\% $\uparrow$} & \textbf{31.65} & \textbf{22.06\% $\uparrow$} & \textbf{31.51} & \textbf{14.04\% $\uparrow$} \\
\midrule
\rowcolor[HTML]{f0f0f0}
\multicolumn{11}{l}{\textit{\textbf{Policy Model: Qwen2.5-7B-Instruct}}} \\
\rowcolor{white}
Base Model & 68.98 & -- & 27.79 & -- & 30.08 & -- & 51.90 & -- & 53.41 & -- \\
Uniform & 71.83 & 4.13\% $\uparrow$ & 40.27 & 44.91\% $\uparrow$ & 33.83 & 12.47\% $\uparrow$ & 56.20 & 8.29\% $\uparrow$ & 56.88 & 6.50\% $\uparrow$ \\
RACE-SC & 72.12 & 4.55\% $\uparrow$ & 40.42 & 45.45\% $\uparrow$ & 36.19 & 20.31\% $\uparrow$ & 57.08 & 9.98\% $\uparrow$ & 57.68 & 7.99\% $\uparrow$ \\
RRD-WU & 72.48 & 5.07\% $\uparrow$ & 41.54 & 49.48\% $\uparrow$ & 36.48 & 21.28\% $\uparrow$ & 55.49 & 6.92\% $\uparrow$ & 55.99 & 4.83\% $\uparrow$ \\
\rowcolor[HTML]{ecf0ff}
\textbf{ICW} & \textbf{73.84} & \textbf{7.05\% $\uparrow$} & \textbf{42.83} & \textbf{54.12\% $\uparrow$} & \textbf{37.66} & \textbf{25.20\% $\uparrow$} & \textbf{59.78} & \textbf{15.18\% $\uparrow$} & \textbf{58.05} & \textbf{8.69\% $\uparrow$} \\
\midrule
\rowcolor[HTML]{f0f0f0}
\multicolumn{11}{l}{\textit{\textbf{Policy Model: Qwen3-4B}}} \\
\rowcolor{white}
Base Model & 76.79 & -- & 51.87 & -- & 42.14 & -- & 73.52 & -- & 72.27 & -- \\
Uniform & 77.73 & 1.22\% $\uparrow$ & 54.37 & 4.82\% $\uparrow$ & 43.70 & 3.70\% $\uparrow$ & 74.46 & 1.28\% $\uparrow$ & 74.21 & 2.68\% $\uparrow$ \\
RACE-SC & 79.28 & 3.24\% $\uparrow$ & 56.43 & 8.79\% $\uparrow$ & 45.14 & 7.12\% $\uparrow$ & 74.28 & 1.03\% $\uparrow$ & 73.32 & 1.45\% $\uparrow$ \\
RRD-WU & 79.17 & 3.10\% $\uparrow$ & 56.87 & 9.64\% $\uparrow$ & 45.81 & 8.71\% $\uparrow$ & 75.64 & 2.88\% $\uparrow$ & 74.44 & 3.00\% $\uparrow$ \\
\rowcolor[HTML]{ecf0ff}
\textbf{ICW} & \textbf{79.80} & \textbf{3.92\% $\uparrow$} & \textbf{56.98} & \textbf{9.85\% $\uparrow$} & \textbf{46.39} & \textbf{10.09\% $\uparrow$} & \textbf{76.46} & \textbf{4.00\% $\uparrow$} & \textbf{75.21} & \textbf{4.07\% $\uparrow$} \\
\bottomrule
\end{tabular}
\begin{tablenotes}
\item \scriptsize \textit{Note}: ``WR'', ``LC'', and ``SC'' denote the win rate, the length-controlled win rate, and the style-controlled win rate, respectively, and ``RI'' denotes relative improvement over Base Model.
\end{tablenotes}
\end{threeparttable}
\end{table}

In this section, we investigate whether the improved evaluation scores of ICW, which result from accounting for the unequal importance of criteria, translate into better policy models in downstream RL.
We select representative baselines, use the evaluation scores of each method as rewards on user prompts from WildChat~\citep{zhao2024wildchat}, and fine-tune three policy models from the Qwen2.5~\citep{qwen2024qwen25} and Qwen3~\citep{yang2025qwen3} families via GRPO~\citep{shao2024deepseekmath}.
We further report the policy models before RL, referred to as Base Model, as a reference.
We evaluate the policy models on BiGGen Bench~\citep{kim2025biggen}, AlpacaEval 2.0~\citep{dubois2024length}, and Arena-Hard~\citep{li2025arenahard}, and more details are provided in Appendix~\ref{app:implementation}.
The results are presented in Table~\ref{tab:downstream_rl} with key observations as follows.
\ding{182} \textbf{Uniform weighting also limits downstream RL performance.}
RACE-SC and RRD-WU, which achieve higher accuracy than Uniform in Table~\ref{tab:main_result}, generally outperform Uniform.
This indicates that treating all criteria equally produces misleading rewards, which in turn limit policy optimization.
\ding{183} \textbf{ICW achieves the best overall downstream RL performance.}
ICW achieves the best results on nearly all benchmarks across policy models; for example, it improves the AlpacaEval LC win rate of Qwen2.5-3B-Instruct by 40.87\% over Base Model, compared with 14.12\% for Uniform.
These results demonstrate that weights reflecting the relative importance of the criteria provide more effective rewards.

\section{Conclusion}
\label{sec:conclusion}

In this paper, we investigated the rubric weighting problem in rubric-based evaluation, where not all criteria matter equally but existing methods commonly aggregate the verification results with uniform weighting.
To address this limitation, we proposed ICW, which derives the weights of criteria from pairwise importance comparisons made by an LLM, and checks and revises inconsistent comparisons.
Extensive experiments confirm that ICW consistently outperforms uniform weighting and competitive weighting methods, improving the pairwise preference accuracy over uniform weighting by up to 5.79\%, and translates into better policy models in downstream RL---including a 40.87\% improvement in the AlpacaEval LC win rate of Qwen2.5-3B-Instruct over Base Model, compared with 14.12\% for uniform weighting.

\textbf{Limitations \& Future Work.}
Two limitations of this work warrant future investigation: (1) ICW compares every pair of criteria, so the number of queries grows quadratically with the number of criteria (Appendix~\ref{app:cost}), motivating future work on comparing only informative pairs. (2) In specialized domains such as healthcare, the weighting model may lack the knowledge required to judge the importance of criteria. Incorporating expert-specified comparisons remains a valuable direction.

\bibliography{iclr2027_conference,arxiv_references}
\bibliographystyle{iclr2027_conference}

\newpage
\appendix

\section{Related Work}
\label{app:related_work}

\textbf{Rubric-Based Evaluation.}
Rubric-based evaluation is increasingly adopted to evaluate open-ended LLM responses, where response quality is decomposed into explicit criteria verified by LLMs~\citep{hashemi2024llmrubric,li2025generation}.
Existing works mainly advance the construction, coverage, and downstream use of rubrics.
For construction, rubrics are written by experts~\citep{arora2025healthbench,zhang2026rubricbench}, mined from survey articles~\citep{yifei2026researchqa}, or generated by LLMs for each user prompt~\citep{cook2024tick,li2026rubrichub,xu2026rubricarm}.
For coverage, rubrics are extended with criteria elicited from pairwise comparisons of responses~\citep{rezaei2025onlinerubrics} or generated to comprehensively capture the evaluation requirements of a user prompt~\citep{chen2026genrubric}.
Recent work further evolves rubric sets: Rubrics on Trial~\citep{yang2026rubricsontrial} uses synthetic response-pair comparisons to assess criterion utility and necessity, while RubricArmor~\citep{yang2026rubricarmor} alternates adversarial response generation with rubric repair to address omitted or underspecified criteria.
For downstream use, evaluation scores are used to rank candidate responses~\citep{li2026rubrichub} and serve as rewards for RL beyond verifiable domains~\citep{mu2024rulebased,gunjal2026rubrics,viswanathan2025checklists,zhang2026chasing,shan2026rubricrlsurvey}.
Aggregation schemes include both uniform weighting~\citep{qin2024infobench,cook2024tick,lee2025checkeval} and explicit importance weights~\citep{arora2025healthbench,viswanathan2025checklists,li2026rubrichub}; we focus on how these weights are obtained.

\textbf{Rubric Weighting.}
Beyond uniform weighting, some approaches rely on experts to specify criterion weights~\citep{arora2025healthbench}, which is costly and difficult to scale.
Existing automatic methods can be divided into direct and indirect weighting methods.
Direct weighting methods prompt LLMs to assign numerical weights or importance levels to individual criteria or evaluation dimensions~\citep{du2025deepresearch,viswanathan2025checklists,gunjal2026rubrics,zhang2026chasing,ding2026adarubric,li2026rubrichub}.
Some of them stabilize the weights by aggregating repeated generations~\citep{du2025deepresearch} or assigning coarse importance levels~\citep{gunjal2026rubrics,zhang2026chasing}, while still eliciting importance judgments directly. Findings of anchoring effects and inconsistent ratings in text evaluation~\citep{stureborg2024large} motivate examining alternative elicitation schemes.
Indirect weighting methods derive the weights from pairwise comparisons or verification results.
MAPLE~\citep{jinji2026maple} compares criteria pairwise with multiple LLMs, but aggregates the comparisons with a Bradley--Terry model~\citep{bradley1952rank} without checking their consistency.
RRD~\citep{shen2026rethinking} downweights correlated criteria based on the verification results of sampled responses, which reduces repeated credit for overlapping criteria but disregards the importance of the criteria.
In contrast, ICW derives the weights from pairwise importance comparisons through the principal eigenvector, and checks and revises the comparisons based on their consistency.

\section{Theoretical Justification}
\label{app:proofs}

We provide the proofs of Propositions~\ref{prop:eigenvector} and~\ref{prop:consistency} and the derivation of Eq.~(\ref{eq:consistency_index}) following~\citet{saaty1977scaling}, and further show in Proposition~\ref{prop:order} that the principal eigenvector preserves the order of transitive judgments.

\subsection{Proof of Proposition~\ref{prop:eigenvector}}

\begin{proof}
If there exist weights $\boldsymbol{w}$ with $w_k>0$ such that $a_{ij}=w_i/w_j$ for all $i,j$, then $a_{ij}a_{j\ell}=(w_i/w_j)(w_j/w_\ell)=w_i/w_\ell=a_{i\ell}$ for all $i,j,\ell$, so $\mathbf{A}$ is consistent.
Conversely, let $\mathbf{A}$ be consistent, and define $w_i=a_{i1}/\sum_{k=1}^{K}a_{k1}$ for $i=1,\ldots,K$, so that $w_i>0$ and $\sum_{i=1}^{K}w_i=1$.
For any $i$ and $j$, consistency and reciprocity give $a_{ij}=a_{i1}a_{1j}=a_{i1}/a_{j1}=w_i/w_j$.
If another normalized vector $\boldsymbol{w}'$ also satisfies $a_{ij}=w'_i/w'_j$ for all $i,j$, then $w'_i/w'_1=a_{i1}=w_i/w_1$ for all $i$, so $\boldsymbol{w}'$ is proportional to $\boldsymbol{w}$, and the normalization gives $\boldsymbol{w}'=\boldsymbol{w}$.
Moreover, for each $i$,
\begin{equation*}
    (\mathbf{A}\boldsymbol{w})_i=\sum_{j=1}^{K}a_{ij}w_j=\sum_{j=1}^{K}\frac{w_i}{w_j}w_j=Kw_i,
\end{equation*}
so $\mathbf{A}\boldsymbol{w}=K\boldsymbol{w}$.
Since $\mathbf{A}$ is positive and $\boldsymbol{w}$ is a positive eigenvector, the Perron--Frobenius theorem~\citep{horn2012matrix} implies that its eigenvalue is the largest eigenvalue of $\mathbf{A}$, i.e., $\lambda_{\max}=K$, and that $\boldsymbol{w}$ is the principal eigenvector.
\end{proof}

\subsection{Proof of Proposition~\ref{prop:consistency}}

\begin{proof}
By the Perron--Frobenius theorem~\citep{horn2012matrix}, the principal eigenvector $\boldsymbol{w}$ of $\mathbf{A}$ is positive.
Define $\epsilon_{ij}=a_{ij}w_j/w_i>0$ for all $i,j$, so that reciprocity gives $\epsilon_{ii}=1$ and $\epsilon_{ji}=1/\epsilon_{ij}$.
The $i$-th row of $\mathbf{A}\boldsymbol{w}=\lambda_{\max}\boldsymbol{w}$ gives $\lambda_{\max}=\sum_{j=1}^{K}\epsilon_{ij}$ for each $i$.
Summing over $i$ yields
\begin{equation*}
    K\lambda_{\max}=\sum_{i=1}^{K}\sum_{j=1}^{K}\epsilon_{ij}=K+\sum_{1\le i<j\le K}\left(\epsilon_{ij}+\frac{1}{\epsilon_{ij}}\right)\geq K+K(K-1)=K^2,
\end{equation*}
where the inequality follows from $t+1/t\geq2$ for $t>0$, with equality if and only if $t=1$.
Hence, $\lambda_{\max}\geq K$, with equality if and only if $\epsilon_{ij}=1$, i.e., $a_{ij}=w_i/w_j$, for all $i,j$.
If the equality holds, $\mathbf{A}$ is consistent by Proposition~\ref{prop:eigenvector}; conversely, if $\mathbf{A}$ is consistent, Proposition~\ref{prop:eigenvector} gives $\lambda_{\max}=K$.
Moreover, since there are $K(K-1)/2$ pairs with $i<j$, subtracting $K^2$ from both sides of $K\lambda_{\max}=K+\sum_{i<j}(\epsilon_{ij}+1/\epsilon_{ij})$ and dividing by $K(K-1)$ yields Eq.~(\ref{eq:consistency_index}).
\end{proof}

\subsection{Order Preservation of the Principal Eigenvector}
\label{app:order}

We call the judgments transitive if the criteria can be assigned importance levels $r_1,\ldots,r_K$ such that $a_{ij}=\alpha$ if $r_i>r_j$, $a_{ij}=1$ if $r_i=r_j$, and $a_{ij}=1/\alpha$ if $r_i<r_j$, i.e., the comparisons contain no contradiction such as a cycle.

\begin{proposition}
\label{prop:order}
Let $\mathbf{A}$ be the comparison matrix in Eq.~(\ref{eq:comparison}) with $\alpha>1$, and suppose that the judgments are transitive with importance levels $r_1,\ldots,r_K$.
Then the principal eigenvector $\boldsymbol{w}$ of $\mathbf{A}$ satisfies $w_i>w_j$ if $r_i>r_j$, and $w_i=w_j$ if $r_i=r_j$.
\end{proposition}

\begin{proof}
Let $s(t)=\alpha$ if $t>0$, $s(t)=1$ if $t=0$, and $s(t)=1/\alpha$ if $t<0$, so that $a_{ik}=s(r_i-r_k)$ for all $i,k$, and $s$ is nondecreasing.
If $r_i>r_j$, then $r_i-r_k>r_j-r_k$ gives $a_{ik}\geq a_{jk}$ for every $k$, with strict inequality for $k=j$ since $a_{ij}=\alpha>1=a_{jj}$.
Since $\boldsymbol{w}$ is positive by the Perron--Frobenius theorem~\citep{horn2012matrix}, the $i$-th and $j$-th rows of $\mathbf{A}\boldsymbol{w}=\lambda_{\max}\boldsymbol{w}$ give
\begin{equation*}
    \lambda_{\max}w_i=\sum_{k=1}^{K}a_{ik}w_k>\sum_{k=1}^{K}a_{jk}w_k=\lambda_{\max}w_j,
\end{equation*}
so $w_i>w_j$ as $\lambda_{\max}>0$.
If $r_i=r_j$, then $a_{ik}=a_{jk}$ for every $k$, and the same two rows give $w_i=w_j$.
\end{proof}

Proposition~\ref{prop:order} holds for any $\alpha>1$, whereas $\mathbf{A}$ is inconsistent whenever the importance levels take more than two distinct values, e.g., $a_{ij}a_{j\ell}=\alpha^2\neq\alpha=a_{i\ell}$ for $r_i>r_j>r_\ell$.
Hence, the inconsistency caused by encoding transitive judgments with the fixed ratio does not change the order of the weights.

\section{More Experimental Results}
\label{app:more_results}

In this section, we present supplementary empirical results deferred from the main text. First, we extend the generalization study of Section~\ref{subsec:generalization} to RubricBench, ResearchQA, and JudgeBench (Section~\ref{app:generalization}). Second, we extend the sensitivity study of Section~\ref{subsec:hyperparameter} to other weighting models (Section~\ref{app:sensitivity}). Third, we report the token cost of ICW (Section~\ref{app:cost}).

\subsection{More Generalization Results}
\label{app:generalization}

In this subsection, we extend the generalization study of Section~\ref{subsec:generalization} to RubricBench, ResearchQA, and JudgeBench.
The results are presented in Tables~\ref{tab:gen_verifier_full} and~\ref{tab:gen_weighting_model_full}, and in line with the findings in Section~\ref{subsec:generalization}, we make two observations.
\ding{182} \textbf{ICW generalizes across verifiers.}
With GPT-4.1 as the weighting model, ICW generally outperforms Uniform across the five verifiers on the three benchmarks.
This further confirms that the weights of ICW complement the verification results of different verifiers.
\ding{183} \textbf{ICW generalizes across weighting models.}
With Gemini-3.1-Flash-Lite as the verifier, ICW consistently outperforms Uniform with all five weighting models on the three benchmarks.
This further confirms that pairwise importance comparison does not rely on a particular weighting model.

\begin{table}[ht]
\begin{minipage}[t]{0.49\linewidth}
\centering
\caption{Results across verifiers on the other benchmarks.}
\label{tab:gen_verifier_full}
\belowrulesep=1pt
\aboverulesep=0.5pt
\setlength{\tabcolsep}{3.44pt}
\renewcommand{\arraystretch}{1}
\scriptsize
\begin{threeparttable}
\begin{tabular}{lcccccc}
\toprule
\multirow{2}{*}{\textbf{Method}} & \multicolumn{2}{c}{\textbf{RubricBench}} & \multicolumn{2}{c}{\textbf{ResearchQA}} & \multicolumn{2}{c}{\textbf{JudgeBench}} \\
\cmidrule(lr){2-3} \cmidrule(lr){4-5} \cmidrule(lr){6-7}
& Acc & RI & Acc & RI & Acc & RI \\
\midrule
\rowcolor[HTML]{f0f0f0}
\multicolumn{7}{l}{\textbf{\textit{Verifier: Gemini-3.5-Flash}}} \\
\rowcolor{white}
Uniform & 0.795 & -- & 0.601 & -- & 0.802 & -- \\
\rowcolor[HTML]{ecf0ff}
\textbf{ICW} & \textbf{0.804} & \textbf{1.16\% $\uparrow$} & \textbf{0.612} & \textbf{1.96\% $\uparrow$} & \textbf{0.830} & \textbf{3.52\% $\uparrow$} \\
\midrule
\rowcolor[HTML]{f0f0f0}
\multicolumn{7}{l}{\textbf{\textit{Verifier: GPT-4.1-mini}}} \\
\rowcolor{white}
Uniform & 0.760 & -- & \textbf{0.583} & -- & 0.589 & -- \\
\rowcolor[HTML]{ecf0ff}
\textbf{ICW} & \textbf{0.773} & \textbf{1.72\% $\uparrow$} & 0.580 & 0.50\% $\downarrow$ & \textbf{0.607} & \textbf{3.02\% $\uparrow$} \\
\midrule
\rowcolor[HTML]{f0f0f0}
\multicolumn{7}{l}{\textbf{\textit{Verifier: DeepSeek-V4-Flash}}} \\
\rowcolor{white}
Uniform & 0.788 & -- & 0.592 & -- & 0.704 & -- \\
\rowcolor[HTML]{ecf0ff}
\textbf{ICW} & \textbf{0.792} & \textbf{0.56\% $\uparrow$} & \textbf{0.604} & \textbf{2.01\% $\uparrow$} & \textbf{0.715} & \textbf{1.49\% $\uparrow$} \\
\midrule
\rowcolor[HTML]{f0f0f0}
\multicolumn{7}{l}{\textbf{\textit{Verifier: GLM-5.1}}} \\
\rowcolor{white}
Uniform & 0.791 & -- & 0.604 & -- & 0.759 & -- \\
\rowcolor[HTML]{ecf0ff}
\textbf{ICW} & \textbf{0.796} & \textbf{0.66\% $\uparrow$} & \textbf{0.618} & \textbf{2.44\% $\uparrow$} & \textbf{0.775} & \textbf{2.12\% $\uparrow$} \\
\midrule
\rowcolor[HTML]{f0f0f0}
\multicolumn{7}{l}{\textbf{\textit{Verifier: Gemini-3.1-Flash-Lite}}} \\
\rowcolor{white}
Uniform & 0.771 & -- & 0.583 & -- & 0.762 & -- \\
\rowcolor[HTML]{ecf0ff}
\textbf{ICW} & \textbf{0.795} & \textbf{3.11\% $\uparrow$} & \textbf{0.612} & \textbf{5.08\% $\uparrow$} & \textbf{0.783} & \textbf{2.76\% $\uparrow$} \\
\bottomrule
\end{tabular}
\begin{tablenotes}
\item \tiny \textit{Note}: ``RI'' denotes relative improvement over the Uniform baseline.
\end{tablenotes}
\end{threeparttable}
\end{minipage}
\hfill
\begin{minipage}[t]{0.49\linewidth}
\centering
\caption{Results across weighting models on the other benchmarks.}
\label{tab:gen_weighting_model_full}
\belowrulesep=1pt
\aboverulesep=0.5pt
\setlength{\tabcolsep}{3.44pt}
\renewcommand{\arraystretch}{1}
\scriptsize
\begin{threeparttable}
\begin{tabular}{lcccccc}
\toprule
\multirow{2}{*}{\textbf{Method}} & \multicolumn{2}{c}{\textbf{RubricBench}} & \multicolumn{2}{c}{\textbf{ResearchQA}} & \multicolumn{2}{c}{\textbf{JudgeBench}} \\
\cmidrule(lr){2-3} \cmidrule(lr){4-5} \cmidrule(lr){6-7}
& Acc & RI & Acc & RI & Acc & RI \\
\midrule
\rowcolor[HTML]{f0f0f0}
\multicolumn{7}{l}{\textbf{\textit{Weighting Model: Gemini-3.5-Flash}}} \\
\rowcolor{white}
Uniform & 0.771 & -- & 0.583 & -- & 0.762 & -- \\
\rowcolor[HTML]{ecf0ff}
\textbf{ICW} & \textbf{0.792} & \textbf{2.83\% $\uparrow$} & \textbf{0.615} & \textbf{5.58\% $\uparrow$} & \textbf{0.782} & \textbf{2.65\% $\uparrow$} \\
\midrule
\rowcolor[HTML]{f0f0f0}
\multicolumn{7}{l}{\textbf{\textit{Weighting Model: GPT-4.1-mini}}} \\
\rowcolor{white}
Uniform & 0.771 & -- & 0.583 & -- & 0.762 & -- \\
\rowcolor[HTML]{ecf0ff}
\textbf{ICW} & \textbf{0.789} & \textbf{2.44\% $\uparrow$} & \textbf{0.607} & \textbf{4.07\% $\uparrow$} & \textbf{0.783} & \textbf{2.76\% $\uparrow$} \\
\midrule
\rowcolor[HTML]{f0f0f0}
\multicolumn{7}{l}{\textbf{\textit{Weighting Model: DeepSeek-V4-Flash}}} \\
\rowcolor{white}
Uniform & 0.771 & -- & 0.583 & -- & 0.762 & -- \\
\rowcolor[HTML]{ecf0ff}
\textbf{ICW} & \textbf{0.795} & \textbf{3.11\% $\uparrow$} & \textbf{0.601} & \textbf{3.05\% $\uparrow$} & \textbf{0.777} & \textbf{2.01\% $\uparrow$} \\
\midrule
\rowcolor[HTML]{f0f0f0}
\multicolumn{7}{l}{\textbf{\textit{Weighting Model: GLM-5.1}}} \\
\rowcolor{white}
Uniform & 0.771 & -- & 0.583 & -- & 0.762 & -- \\
\rowcolor[HTML]{ecf0ff}
\textbf{ICW} & \textbf{0.797} & \textbf{3.40\% $\uparrow$} & \textbf{0.607} & \textbf{4.07\% $\uparrow$} & \textbf{0.782} & \textbf{2.65\% $\uparrow$} \\
\midrule
\rowcolor[HTML]{f0f0f0}
\multicolumn{7}{l}{\textbf{\textit{Weighting Model: GPT-4.1}}} \\
\rowcolor{white}
Uniform & 0.771 & -- & 0.583 & -- & 0.762 & -- \\
\rowcolor[HTML]{ecf0ff}
\textbf{ICW} & \textbf{0.795} & \textbf{3.11\% $\uparrow$} & \textbf{0.612} & \textbf{5.08\% $\uparrow$} & \textbf{0.783} & \textbf{2.76\% $\uparrow$} \\
\bottomrule
\end{tabular}
\begin{tablenotes}
\item \tiny \textit{Note}: ``RI'' denotes relative improvement over the Uniform baseline.
\end{tablenotes}
\end{threeparttable}
\end{minipage}
\end{table}

\needspace{14\baselineskip}
\subsection{More Sensitivity Results}
\label{app:sensitivity}

\begin{wrapfigure}{r}{0.5\linewidth}
\vspace{-2em}
\centering
\includegraphics[width=\linewidth]{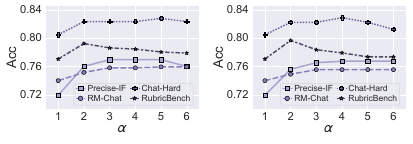}
\vspace{-1em}
\caption{Sensitivity to the fixed ratio $\alpha$ with Gemini-3.5-Flash (left) and GLM-5.1 (right) as the weighting models.}
\label{fig:alpha_weight_models}
\vspace{-5pt}
\end{wrapfigure}

In this subsection, we extend the sensitivity study of Section~\ref{subsec:hyperparameter} to Gemini-3.5-Flash and GLM-5.1 as the weighting models.
The results are presented in Figure~\ref{fig:alpha_weight_models}, and in line with the findings in Section~\ref{subsec:hyperparameter}, we make two observations.
\ding{182} \textbf{ICW improves over uniform weighting across weighting models.}
With both weighting models, ICW outperforms uniform weighting ($\alpha=1$) on all four benchmarks for any tested $\alpha\geq2$.
This indicates that the benefit of pairwise importance comparison does not depend on a particular weighting model.
\ding{183} \textbf{ICW remains robust to the fixed ratio across weighting models.}
Once $\alpha\geq2$, the accuracy remains stable as $\alpha$ further increases for both weighting models.
This confirms that the robustness to $\alpha$ observed with GPT-4.1 in Section~\ref{subsec:hyperparameter} is not tied to a particular weighting model.

\subsection{Cost Analysis}
\label{app:cost}

Table~\ref{tab:cost} reports the token cost of ICW, i.e., the average numbers of input and output tokens consumed by the weighting model per sample.

\begin{table}[ht]
\centering
\caption{Token cost of ICW with different weighting models.}
\label{tab:cost}
\setlength{\tabcolsep}{7.50pt}
\renewcommand{\arraystretch}{1}
\scriptsize
\begin{threeparttable}
\begin{tabular}{llcccccc}
\toprule
\textbf{Weighting Model} & \textbf{Type} & \textbf{RubricBench} & \textbf{ResearchQA} & \textbf{JudgeBench} & \textbf{RM-Chat} & \textbf{Chat-Hard} & \textbf{Precise-IF} \\
\midrule
\multirow{2}{*}{GPT-4.1} & Input & 39.552 & 152.116 & 111.641 & 92.567 & 107.761 & 185.989 \\
 & Output & 1.377 & 2.957 & 1.962 & 1.799 & 1.872 & 2.510 \\
\midrule
\multirow{2}{*}{Gemini-3.5-Flash} & Input & 49.799 & 190.582 & 127.352 & 106.721 & 122.638 & 231.673 \\
 & Output & 18.074 & 41.253 & 22.049 & 17.428 & 17.770 & 29.697 \\
\midrule
\multirow{2}{*}{GLM-5.1} & Input & 61.862 & 180.468 & 113.025 & 94.347 & 110.812 & 197.559 \\
 & Output & 26.279 & 53.043 & 26.773 & 22.474 & 23.030 & 33.580 \\
\bottomrule
\end{tabular}
\begin{tablenotes}
\item \scriptsize \textit{Note}: The numbers are in thousands of tokens per sample, and the output tokens include reasoning tokens.
\end{tablenotes}
\end{threeparttable}
\end{table}

\section{Reproduction Details}
\label{app:reproduction}

\subsection{Dataset Descriptions}
\label{app:dataset}

In this section, we document the datasets adopted in our experiments, which serve two purposes: (1) pairwise preference benchmarks for evaluating rubric weighting methods, and (2) training prompts and evaluation benchmarks for downstream RL.

\paragraph{Pairwise Preference Benchmarks.}
Each benchmark provides user prompts, pairs of candidate responses, and preference labels.
\begin{itemize}[leftmargin=*]
    \item \textbf{RubricBench}~\citep{zhang2026rubricbench}\footnote{\url{https://github.com/planepig/rubricbench}}: A benchmark for rubric-guided evaluation that targets hard samples with nuanced input complexity and misleading surface features, where each sample is annotated with expert-written rubrics derived from the instruction.
    \item \textbf{ResearchQA}~\citep{yifei2026researchqa}\footnote{\url{https://huggingface.co/datasets/realliyifei/ResearchQA}}: A scholarly question answering benchmark distilled from survey articles across research fields, where each query is paired with query-specific rubric items. We use the response pairs with human preference annotations.
    \item \textbf{JudgeBench}~\citep{tan2025judgebench}\footnote{\url{https://github.com/ScalerLab/JudgeBench}}: A benchmark of challenging response pairs spanning knowledge, reasoning, math, and coding, where the preference labels reflect factual and logical correctness.
    \item \textbf{RM-Bench}~\citep{liu2025rmbench}\footnote{\url{https://huggingface.co/datasets/THU-KEG/RM-Bench}}: A reward model benchmark that evaluates sensitivity to subtle content differences and resistance to style biases. We use its Chat subset.
    \item \textbf{RewardBench}~\citep{lambert2025rewardbench}\footnote{\url{https://huggingface.co/datasets/allenai/reward-bench}}: A reward model benchmark spanning chat, reasoning, and safety. We use its Chat-Hard subset, which contains challenging chat pairs with subtle differences.
    \item \textbf{RewardBench 2}~\citep{malik2026rewardbench2}\footnote{\url{https://huggingface.co/datasets/allenai/reward-bench-2}}: A reward model benchmark built on new human prompts. We use its Precise-IF subset, which evaluates whether responses follow precise instruction constraints.
\end{itemize}
For JudgeBench, RM-Chat, Chat-Hard, and Precise-IF, the rubric sets are generated by the RubricARM-8B-Rubric generator of Rubric-ARM~\citep{xu2026rubricarm}\footnote{\url{https://huggingface.co/OpenRubrics/RubricARM-8B-Rubric}}.

\paragraph{Downstream RL Data and Benchmarks.}
\begin{itemize}[leftmargin=*]
    \item \textbf{WildChat}~\citep{zhao2024wildchat}: A corpus of real user--ChatGPT conversations. We use its user prompts for RL training, and generate a rubric set for each prompt with Gemini-3.5-Flash.
    \item \textbf{BiGGen Bench}~\citep{kim2025biggen}: A benchmark that evaluates nine capabilities of language models with instance-specific evaluation criteria. We report the score on a 0--100 scale.
    \item \textbf{AlpacaEval 2.0}~\citep{dubois2024length}: A benchmark of single-turn instructions, for which we report the win rate (WR) against reference answers and the length-controlled win rate (LC).
    \item \textbf{Arena-Hard}~\citep{li2025arenahard}: A benchmark of challenging user prompts, for which we report the win rate (WR) against a baseline model and the style-controlled win rate (SC) on Arena-Hard-Auto v0.1.
\end{itemize}

\subsection{Implementation Details}
\label{app:implementation}

In this section, we supplement the implementation details deferred from the main text.

\paragraph{Consistency Check Settings.}
We compute CR in Eq.~(\ref{eq:consistency}) with the random indices $\mathrm{RI}_K$ of~\citet{saaty1980analytic} for $K\le15$.
For $K>15$, we use $\mathrm{RI}_{15}$.

\paragraph{Downstream RL Settings.}
The RL experiments are implemented on top of the \texttt{verl} library\footnote{\url{https://github.com/volcengine/verl}.}.
We adopt GRPO~\citep{shao2024deepseekmath} as the policy optimization algorithm, and initialize the policy models from Qwen2.5-3B-Instruct, Qwen2.5-7B-Instruct~\citep{qwen2024qwen25}, and Qwen3-4B~\citep{yang2025qwen3}, with the thinking mode enabled for Qwen3-4B.
User prompts from WildChat~\citep{zhao2024wildchat} are used for policy training.
All methods share the same rubric sets, and the weights are obtained with GPT-4.1 as the weighting model.
The reward of each response is its evaluation score in Eq.~(\ref{eq:rubric_score}) with GPT-4.1-mini as the verifier.
The training is configured with a batch size of 128, 8 rollouts per prompt, and a sampling temperature of 1.0.
The policy models are optimized with a learning rate of $1\times10^{-6}$ and a KL coefficient of $1\times10^{-3}$ for 3 epochs.
Prompt and response are capped at 2{,}048 and 9{,}216 tokens, respectively.
For evaluation, GPT-4.1 is used as the judge on all benchmarks.

\paragraph{Prompt Templates.}
Prompts~\ref{prompt:pairwise}--\ref{prompt:verifier} present the prompts used in our experiments.
The weighting model compares each criterion pair with Prompt~\ref{prompt:pairwise}, where \{revision\_feedback\} is empty in the initial comparisons and is replaced by Prompt~\ref{prompt:revision} during revision.
The verifier checks each criterion for each response with Prompt~\ref{prompt:verifier}, for both the pairwise preference benchmarks and downstream RL.

\begin{mybox}[label=prompt:pairwise]{Prompt}{Pairwise importance comparison prompt of the weighting model.}
\footnotesize
\setlength{\parskip}{4pt}
\textbf{System message:}
You compare the importance of two rubric criteria for one user prompt. Importance means scoring weight: when grading candidate answers to this prompt, failing a more important criterion must cost more credit. Use only the user prompt and the fixed rubric set. Never guess at candidate answers, human preferences, or verifier outcomes.

\medskip
\textbf{User message:}

User prompt:\\
\{prompt\}

Full rubric set (fixed context for this sample):\\
\{full\_rubric\_set\}

Compare these two rubrics:\\
first (\{first\_rubric\_id\}): \{first\_criterion\}\\
second (\{second\_rubric\_id\}): \{second\_criterion\}

Which highlighted rubric should carry more weight when grading answers to this prompt?

General rules:\\
- First list what the prompt itself fixes: the requested result and every explicit must, must-not, count, format, or procedural rule. A rubric tied to any of them outranks generic quality.\\
- Rank two rubrics only when the prompt gives a reason; never build a preference from plausible reasoning alone. If neither rubric is favored by the prompt, answer "neutral".\\
- Judge the concrete behavior each rubric uniquely requires, not wording such as detailed, complete, robust, or professional.\\
- Use the full rubric set: a duplicate, vague umbrella, or whole-task restatement is less important than a specific unique requirement.\\
- For a negatively phrased rubric, judge the value of avoiding that defect; do not reverse its meaning.\\
- First decide which rubric\_id wins; your answer must not change if the two rubrics were displayed in the opposite order.

\{revision\_feedback\}

Return JSON only:\\
\{"more\_important": "first" | "second" | "neutral", "reason": "<= 25 words"\}
\end{mybox}

\begin{mybox}[label=prompt:revision]{Prompt}{Revision prompt of the weighting model.}
\footnotesize
\setlength{\parskip}{4pt}
Consistency review: the initial judgments for every rubric pair follow. Each forward/reverse entry preserves its display order, winner\_rubric\_id, and original reason. A null winner means that call answered neutral. The status distinguishes:\\
- agreed\_winner: both display orders support the same rubric\_id;\\
- both\_neutral: both calls report no supported importance advantage;\\
- one\_order\_neutral: only one call is neutral;\\
- opposite\_winners: the two calls support different rubric\_ids.\\
The last two statuses indicate order instability. Their null collapsed winner does not establish that the rubrics are equally important. Even both\_neutral can reflect insufficient evidence rather than an established equality.

Initial comparison evidence (fallible prior judgments, not ground truth):\\
\{previous\_round\_history\}

Re-judge the highlighted pair using the user prompt and complete rubric set, checking both previous reasons against the prompt. Use the other pairs as context, not as votes or authoritative evidence. Keep a supported judgment; revise it when its reason is unsupported, misreads a requirement, or overlooks a relevant requirement. Explain the concrete prompt or rubric evidence for your judgment. A request for review does not imply that the highlighted pair's initial judgment was wrong. Seek coherent importance judgments, but do not invent a preference merely to remove a conflict. Answer "neutral" when neither rubric is supported as more important. Return a judgment only for the highlighted pair.
\end{mybox}

\begin{mybox}[label=prompt:verifier]{Prompt}{Prompt of the verifier.}
\footnotesize
\setlength{\parskip}{4pt}
\textbf{System message:}
You are a strict rubric verifier. Given one user prompt, one response, and ONE evaluation criterion, decide whether the response satisfies that criterion. Judge only this local criterion, not overall response quality, and never compare against another response. Mark satisfied=1 only when the response clearly and completely satisfies the criterion; partial satisfaction is 0. Return strict JSON only.

\medskip
\textbf{User message:}

Original user prompt:\\
\{prompt\}

Response to verify:\\
\{response\}

Single criterion:\\
\{criterion\}

Verification rules:\\
- Judge only the criterion's stated local boundary. Do not fail it merely because another part of the overall task is weak or missing.\\
- Preserve exact quantifiers: exactly, at least, at most, only, every, and each are not interchangeable.\\
- If the criterion forbids or says to avoid something, satisfied=1 only when the prohibited condition is absent.\\
- Judge meaning and response evidence, not keyword overlap.\\
- The reason must be one concise final justification. The satisfied bit and reason must agree.

Return exactly:\\
\{"satisfied": 1, "reason": "short evidence-based reason"\}
\end{mybox}

\end{document}